\documentclass[11pt]{article}

\usepackage[final]{acl}

\usepackage{times}
\usepackage{latexsym}

\usepackage[T1]{fontenc}
\usepackage[utf8]{inputenc}

\usepackage{microtype}

\usepackage{dblfloatfix}

\usepackage[table]{xcolor}

\definecolor{oursrow}{HTML}{F3F8FF}      % very light blue
\definecolor{bestcell}{HTML}{DFF3E3}     % light green
\definecolor{secondcell}{HTML}{FFF3CD}   % light yellow

\newcommand{\best}[1]{\cellcolor{bestcell}\textbf{#1}}
\newcommand{\second}[1]{\cellcolor{secondcell}\underline{#1}}

\usepackage{graphicx}

\usepackage{amsmath}
\usepackage{amsthm}
\usepackage{amssymb}
\usepackage{amsfonts}
\usepackage{mathtools}

\usepackage{multirow}
\usepackage{wrapfig}
\usepackage{subcaption}
\usepackage{float}

\usepackage{hyperref}
\usepackage{url}

\usepackage{algorithm}
\usepackage{algpseudocode}

\newtheorem{proposition}{Proposition}

\usepackage{booktabs}
\usepackage[table]{xcolor}
\usepackage{tabularx}
\usepackage{colortbl}
\usepackage{adjustbox}
\usepackage{enumitem}
\setlist[itemize]{noitemsep}
\usepackage{nicefrac}
\usepackage{fvextra}

\newcommand{\NA}{\cellcolor{black!8}\textsc{N/A}}

\title{Divide by Question, Conquer by Agent: SPLIT-RAG with Question-Driven Graph Partitioning}

\author{
  Ruiyi Yang\textsuperscript{1} \quad
  Hao Xue\textsuperscript{2,1} \quad
  Imran Razzak\textsuperscript{3,1} \quad
  Shirui Pan\textsuperscript{4} \quad
  Hakim Hacid\textsuperscript{5} \quad
  Flora D. Salim\textsuperscript{1} \\
  \textsuperscript{1}UNSW Sydney, Sydney, NSW, Australia \\
  \textsuperscript{2}The Hong Kong University of Science and Technology
  (Guangzhou), Guangzhou, China \\
  \textsuperscript{3}Mohamed Bin Zayed University of Artificial Intelligence,
  Abu Dhabi, UAE \\
  \textsuperscript{4}Griffith University, Queensland, Australia \\
  \textsuperscript{5}Technology Innovation Institute, Abu Dhabi, UAE \\
  \texttt{\{ruiyi.yang, flora.salim\}@unsw.edu.au};
  \texttt{haoxue@hkust-gz.edu.cn} \\
  \texttt{imran.razzak@mbzuai.ac.ae};
  \texttt{s.pan@griffith.edu.au};
  \texttt{hakim.hacid@tii.ae}
}

\begin{document}
\maketitle

% --------------------------------------------------------------------------
%   Body
% --------------------------------------------------------------------------
\begin{abstract}
Retrieval-augmented generation over knowledge graphs (KGs) can improve the
factual grounding of large language models, but efficient retrieval remains
challenging when the underlying graph is large and multi-hop evidence is
distributed across different regions. We present \textbf{SPLIT-RAG}, an
efficiency-oriented multi-agent KG-RAG framework that organises a KG into
reusable, question-aligned partitions using reasoning patterns observed in
training questions. At inference time, SPLIT-RAG forms a query-specific
retrieval plan, activates only a small set of relevant partition-specialised
agents, and aggregates their retrieved evidence before final answer
generation. To improve reliability when evidence from different agents is
combined, the framework applies compatibility-based filtering over retrieved
triples prior to synthesis. Experiments on WebQSP, CWQ, and MetaQA show that
SPLIT-RAG achieves strong performance against recent KG+LLM baselines while
reducing end-to-end latency under matched budgets. Component ablations and
model-tier analyses further show that selective multi-agent retrieval is
central to the performance gain, and that lightweight retrieval agents paired
with a stronger head agent provide an effective accuracy--latency trade-off.
Code is available at \url{https://github.com/cruiseresearchgroup/SPLIT-RAG}.
\end{abstract}

\section{Introduction}
\label{sec:introduction}

Large language models (LLMs) can produce unsupported or outdated answers when
the required facts are absent from, or no longer reflected in, their
parametric knowledge~\cite{zhang2023siren,kandpal2023large}.
Retrieval-Augmented Generation (RAG) addresses this limitation by grounding
generation in external evidence without requiring model retraining
~\cite{zhu2021retrieving,gao2023retrieval,zhao2024retrieval}.
Knowledge graphs (KGs) are particularly attractive retrieval sources because
their entity--relation structure supports explicit multi-hop reasoning and
structured evidence tracing.

However, scaling KG-based RAG to large graphs remains challenging.
A complex question may require evidence spanning multiple relations, while a
simple question should not require searching a large portion of the graph.
Recent KG+LLM systems improve path or subgraph retrieval, but they still face
an important accuracy--efficiency trade-off: expanding retrieval improves the
chance of covering useful evidence, yet increases latency and introduces more
irrelevant or weakly supported facts into the generation context.
Moreover, when evidence is retrieved from multiple graph regions or agents,
the final generator must combine partially overlapping or potentially
incompatible triples in a controlled manner.

This work studies whether previously observed question-reasoning patterns can
be reused to organise and route KG retrieval.
The central intuition is that training questions reveal recurring retrieval
structures: questions with related entity types and relation paths often depend
on similar graph regions.
Instead of retrieving over the complete KG for every query, a system can
construct reusable question-aligned partitions offline and activate only the
partitions most relevant to a new question at inference time.

\begin{figure*}[t]
\centering
\includegraphics[width=\textwidth]{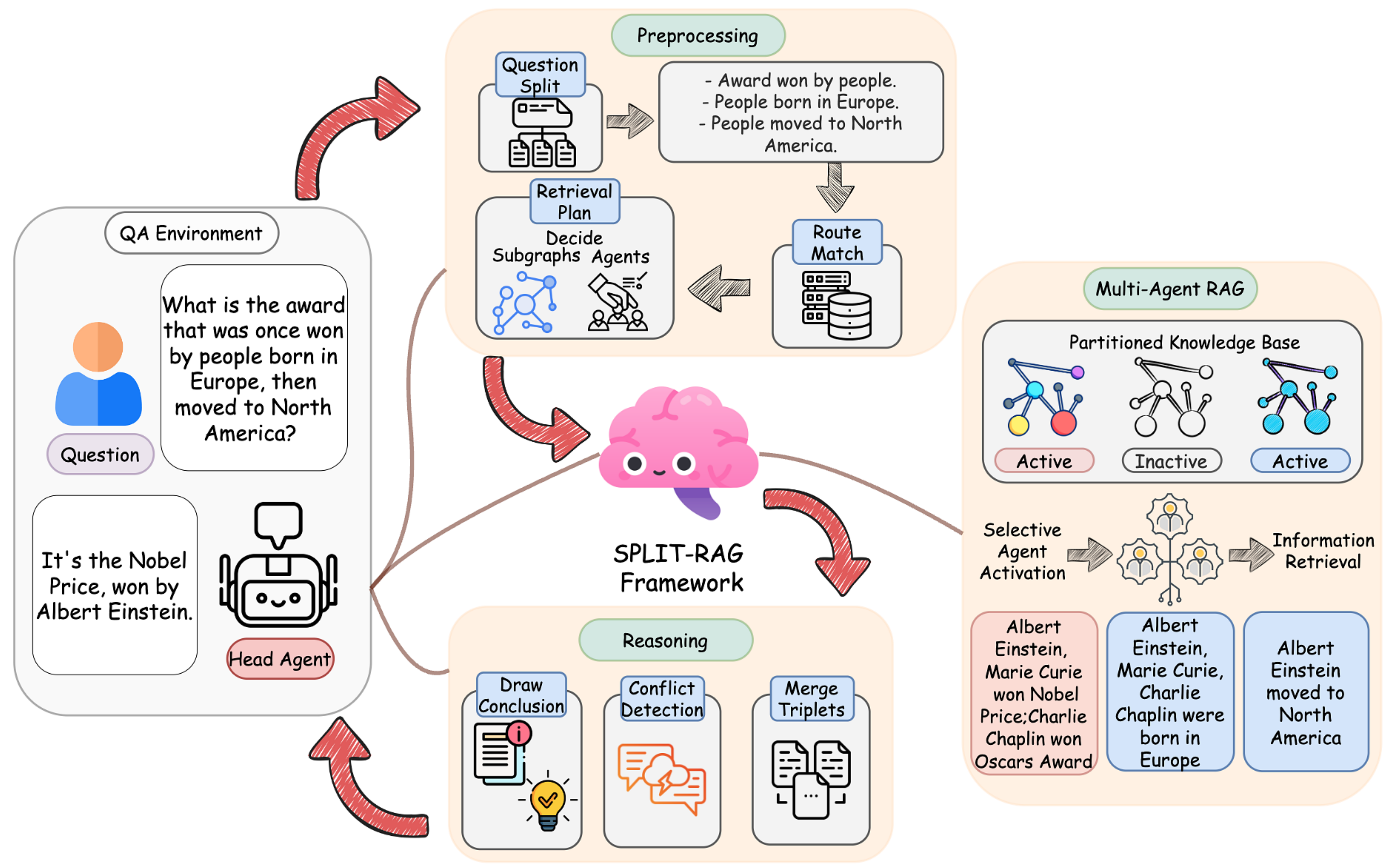}
\caption{
Illustration of SPLIT-RAG.
Training-question reasoning patterns are used to construct question-aligned
graph partitions and assign specialised retrieval agents.
For a new query, the planner selects relevant partition--agent groups,
retrieves supporting triples in parallel, and applies compatibility-based
evidence filtering before head-agent answer generation.
}
\label{fig:split-rag-example}
\end{figure*}

Based on this intuition, we propose \textbf{SPLIT-RAG}, an
efficiency-oriented multi-agent framework for knowledge-graph
retrieval-augmented generation.
SPLIT-RAG first constructs reusable KG partitions from reasoning-path
fragments associated with training-question patterns.
It then assigns partition groups to specialised retrieval agents.
At inference time, a query-specific planner uses entity-type similarity and,
when available, a predicted relation-path sketch to activate a small set of
relevant agents.
The selected agents retrieve evidence only from their assigned partitions in
parallel, after which a head agent generates the answer from a filtered set of
compatible triples and textualised evidence.

This design differs from prior KG+LLM retrieval systems in three respects.
First, SPLIT-RAG uses training-question patterns to construct reusable
retrieval regions before inference, rather than deriving each retrieval space
from the current query alone.
Second, it couples these partitions with specialised agents and performs
query-conditioned partition--agent routing, allowing retrieval effort to be
allocated selectively.
Third, it applies compatibility-based evidence filtering before final
generation, providing a lightweight mechanism for refining evidence collected
from multiple retrieval agents.

Our contributions are as follows:

\begin{itemize}
    \item \textbf{Question-aligned KG partitioning.}
    We introduce an offline partitioning objective that organises a KG into
    compact, reusable regions aligned with recurring training-question
    reasoning patterns, while controlling partition size and fragmentation.

    \item \textbf{Partition-specialised multi-agent retrieval.}
    We develop a routing framework that selects relevant partition--agent
    groups for each incoming question and executes subgraph-specific retrieval
    in parallel, reducing unnecessary graph exploration.

    \item \textbf{Empirical analysis of effectiveness and efficiency.}
    We evaluate SPLIT-RAG on WebQSP, CWQ, and MetaQA against recent KG+LLM
    baselines, and analyse latency, component ablations, model-tier
    assignments, and incomplete-KG settings.
    The results show that selective multi-agent retrieval yields a strong
    accuracy--latency trade-off, with lightweight retrieval agents remaining
    effective when paired with a capable head agent.
\end{itemize}

\section{Related Work}

\begin{table*}[t]
\small
  \centering
  \caption{
  Comparison with representative KG+LLM systems.
  \checkmark{} = present, \texttimes{} = absent, $\circ$ = partial.
  }
  \label{tab:rw-comparison}

  \begin{tabular}{l|ccccc}
    \toprule
    Method & Q-retrieval & Off. part. & MA spec. & Conflict & P-A route \\
    \midrule
    RoG~\cite{luo2023reasoning}
      & \checkmark & \texttimes & \texttimes & \texttimes & \texttimes \\
    ToG~\cite{sun2023think}
      & \checkmark & \texttimes & \texttimes & \texttimes & \texttimes \\
    StructGPT~\cite{jiang2023structgpt}
      & $\circ$ & \texttimes & \texttimes & \texttimes & \texttimes \\
    SubgraphRAG~\cite{li2024simple}
      & \checkmark & \texttimes & \texttimes & \texttimes & \texttimes \\
    Mindful-RAG~\cite{agrawal2024mindful}
      & \checkmark & \texttimes & \texttimes & $\circ$ & \texttimes \\
    GcR~\cite{luo2024graph}
      & \checkmark & \texttimes & \texttimes & \texttimes & \texttimes \\
    GNN-RAG~\cite{mavromatis2025gnn}
      & \checkmark & \texttimes & \texttimes & \texttimes & \texttimes \\
    RPO-RAG~\cite{um2026rpo}
      & \checkmark & \texttimes & \texttimes & \texttimes & \texttimes \\
    \textbf{SPLIT-RAG (ours)}
      & \checkmark & \checkmark & \checkmark & \checkmark & \checkmark \\
    \bottomrule
  \end{tabular}
\end{table*}

\subsection{RAG and Graph-Based Retrieval}

Retrieval-augmented generation (RAG) grounds LLM outputs in external evidence, improving factuality and supporting up-to-date or domain-specific knowledge access~\cite{lewis2020retrieval,jiang2023active,zhao2024retrieval,fan2024survey,gao2023retrieval,yu2024evaluation,pan2024unifying}. 
Recent graph-based RAG methods further exploit entity--relation structures to support interpretable multi-hop retrieval~\cite{peng2024graph,hu2024grag}. Recent approaches further improve KGQA retrieval through efficient graph-neural
retrieval~\cite{mavromatis2025gnn}, differentiable subgraph
retrieval~\cite{gao2025d}, or relation-aware preference optimisation
for small LLMs~\cite{um2026rpo}.
However, retrieving over large graphs can still introduce substantial redundancy and latency. 
SPLIT-RAG addresses this issue by routing each question to a small set of relevant graph partitions rather than searching the full knowledge base.

\subsection{Knowledge Graph Partitioning}

Classical graph partitioning methods optimise structural objectives such as edge-cut or vertex-cut, including Kernighan--Lin, Fiduccia--Mattheyses, METIS, and JA-BE-JA~\cite{bulucc2016recent,kim2011genetic,kernighan1970efficient,fiduccia1988linear,karypis1997metis,rahimian2015distributed}. 
Related KG-scale processing methods consider compression, alignment, landmark-based subgraph construction, or mapping-based partitioning to improve scalability~\cite{qin2020g,trisedya2023align,arenas2024morph}. 
In contrast, SPLIT-RAG partitions the KG according to reasoning paths observed in training questions, so that the resulting subgraphs are aligned with downstream QA retrieval patterns rather than graph structure alone.

\subsection{LLM-Based Multi-Agent Systems}

LLM-based multi-agent systems coordinate specialised agents across software development, engineering, embodied reasoning, simulation, recommendation, and financial decision-making tasks~\cite{li2024survey,dong2024self,wu2024chateda,mehta2023improving,xia2023towards,brohan2023can,huang2022inner,gao2023s3,zhang2024agentcf,guo2023gpt,zhao2023competeai,he2024llm}. 
SPLIT-RAG instantiates this idea for KGQA: agents are assigned to graph partitions, and only agents associated with a query-specific retrieval plan are activated at inference time.

\subsection{Positioning of SPLIT-RAG}
\label{sec:rw-table}

Table~\ref{tab:rw-comparison} positions SPLIT-RAG against representative
KG+LLM systems along five design dimensions.
Existing methods already support query-conditioned retrieval or graph-based
reasoning, including subgraph and path retrieval.
SPLIT-RAG differs in organising the KG into reusable, question-aligned
partitions before inference, assigning specialised retrieval agents to these
partitions, and routing each incoming question to a small set of relevant
partition--agent groups.
It further applies compatibility-based evidence filtering before final answer
generation.

\section{Divide by Question, Conquer by Agent: SPLIT-RAG Framework}
\label{sec:method}

\begin{figure*}[t]
\centering
\includegraphics[width=0.9\textwidth]{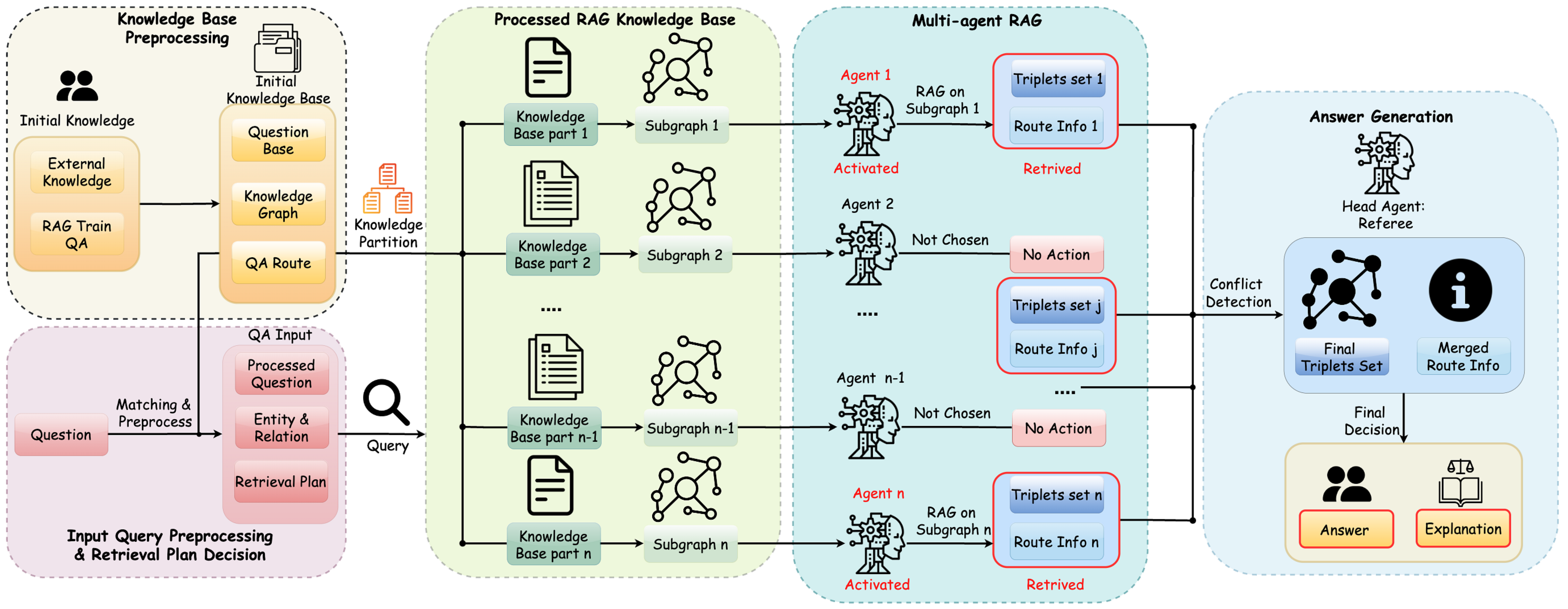}
\caption{Overview of SPLIT-RAG. Training-question reasoning patterns guide graph partitioning and agent assignment. At inference time, a query-specific retrieval plan activates only relevant agents before compatibility-based evidence aggregation and answer generation.}
\label{fig:split-rag-framework}
\end{figure*}

SPLIT-RAG is an efficiency-oriented KG-RAG framework that routes each input question to a small set of specialised retrieval agents over question-aligned graph partitions.
Let \(D=(\mathcal{V},\mathcal{E})\) be a knowledge graph and
\(\mathcal{Q}_{\mathrm{train}}\) a set of training questions with recorded reasoning paths.
An offline indexer \(\tau\) constructs a partitioned knowledge base
\(\hat{\mathcal{K}}=\{\hat{D}_1,\ldots,\hat{D}_M\}\) and assigns partitions to agents.
At inference time, a planner \(\psi\) selects the relevant agents and subgraphs for a new question \(q\), after which a head generator produces the final answer:
\begin{equation}
\label{eq:framework}
\begin{adjustbox}{max width=\linewidth}
\(\displaystyle
\hat{\mathcal{K}}=\tau(\mathcal{Q}_{\mathrm{train}},D),
\qquad
F(q,D)=\hat{\mathcal{G}}\!\left(q,\psi(q,\hat{\mathcal{K}})\right)
\)
\end{adjustbox}
\end{equation}
The framework contains four stages: question-guided graph partitioning, agent allocation, query-time routing and parallel retrieval, and compatibility-based evidence aggregation.

\subsection{Question-Guided Graph Partitioning}
\label{sec:partition}

\paragraph{Question and path representations.}
Each training question \(q_i\) is represented by:
(i) an entity-type view \(\mathbf{Q}_{e,i}\), in which entity mentions are replaced by their KG types; and
(ii) a reasoning path
\[
\mathbf{p}_i=(e_1,r_1,e_2,\ldots,r_{h_i},e_{h_i+1})
\]
that connects entities required for answering \(q_i\).
The entity-type views are clustered into question-pattern groups
\(c_i \in \mathcal{C}\).
Each reasoning path is decomposed into consecutive at-most-two-hop fragments:
\begin{equation}
\phi(\mathbf{p}_i)
=
\left\{
(e_j,r_j,e_{j+1},\ldots,e_{j+\ell})
\,:\,
1\leq \ell\leq 2
\right\}.
\label{eq:path-fragments}
\end{equation}

\paragraph{Partition objective.}
The goal is to construct compact graph regions that preserve frequently reused reasoning patterns.
Let \(Z\) denote the partition assignment of a training-path fragment and let \(C\) denote the question-pattern cluster from which that fragment originates.
A useful partition should make the activated graph region informative of the associated question pattern while avoiding excessively large or fragmented partitions.
We therefore optimise
\begin{equation}
\mathcal{J}(\mathcal{S})
=
I(C;Z)
-
\lambda
\sum_{s\in\mathcal{S}}
\left(\frac{|\mathcal{V}_s|}{|\mathcal{V}|}\right)^2
-
\mu |\mathcal{S}|,
\label{eq:partition-objective}
\end{equation}
subject to
\begin{equation}
|\mathcal{V}_s|\leq \eta_{\max},
\qquad
\forall s\in\mathcal{S},
\qquad
\bigcup_{s\in\mathcal{S}} \mathcal{E}_s = \mathcal{E}.
\label{eq:partition-constraints}
\end{equation}
The mutual-information term rewards partitions associated with coherent question patterns; the quadratic size penalty discourages a small number of oversized partitions; and the final term avoids unnecessary fragmentation.

In practice, path fragments initialise seed regions, which are greedily merged according to the marginal gain in Eq.~\eqref{eq:partition-objective}.
Triples not observed in training reasoning paths are assigned to the nearest compatible region using their entity and relation representations, ensuring that the resulting partition covers the full KG.
Boundary entities may appear in multiple partitions when required to retain cross-partition reasoning paths.
Algorithm~\ref{alg:alg_gp} in Appendix~\ref{app:alg} provides the
partition-construction procedure.

\subsection{Question-Centric Agent Allocation}
\label{sec:allocation}

Let \(\hat{\mathcal{K}}=\{\hat{D}_1,\ldots,\hat{D}_M\}\) be the graph partitions.
For each training question \(q_i\), define its partition coverage set as
\begin{equation}
\label{eq:coverage-indicator}
B_{ij}
=
\mathbb{I}
\!\left[
  \exists\, p\in\textsc{Paths}(q_i):
  p\cap\hat{D}_j\neq\emptyset
\right].
\end{equation}

\begin{equation}
\label{eq:coverage-set}
\mathcal{C}_i
=
\left\{
  \hat{D}_j \mid B_{ij}=1
\right\}.
\end{equation}
Each agent \(a_\ell\) is assigned a group of partitions
\(\mathcal{G}_\ell\subseteq\hat{\mathcal{K}}\).
The allocation objective minimises the number of agents required by training questions:
\begin{equation}
\min_{\{\mathcal{G}_\ell\}_{\ell=1}^{L}}
\;
\mathcal{L}_{\mathrm{coord}}
=
\sum_{i=1}^{|\mathcal{Q}_{\mathrm{train}}|}
\left|
\left\{
\ell :
\mathcal{G}_\ell \cap \mathcal{C}_i \neq \emptyset
\right\}
\right|,
\label{eq:coordination-cost}
\end{equation}
subject to
\begin{equation}
\begin{adjustbox}{max width=\linewidth}
\(\bigcup_{\ell=1}^{L}\mathcal{G}_\ell=\hat{\mathcal{K}},
\qquad
|\mathcal{G}_\ell|\leq N_{\max},
\qquad
\operatorname{Coh}(\mathcal{G}_\ell)\geq\theta_{\mathrm{coh}}\)
\label{eq:allocation-constraints}
\end{adjustbox}
\end{equation}
Here, \(\operatorname{Coh}(\cdot)\) measures semantic similarity among partitions assigned to the same agent.
The objective places partitions frequently used together under the same agent, reducing coordination overhead at inference time.
We use a greedy coverage-co-occurrence allocation procedure, detailed in
Appendix~\ref{app:alg}, Algorithm~\ref{alg:alg_as}.

\subsection{Query-Time Routing and Parallel Retrieval}
\label{sec:routing}

Given a new question \(q\), the planner first forms its entity-type representation \(\mathbf{Q}_{e,q}\).
Optionally, it also predicts a relation-path sketch \(\hat{\mathbf{p}}_q\) directly from the question.
Importantly, \(\hat{\mathbf{p}}_q\) is produced without access to gold test paths or answers; its generation cost is included in end-to-end latency.

For each training question \(q_i\), routing similarity is computed as
\begin{equation}
\label{eq:routing-similarity}
\begin{aligned}
s_{\mathrm{sem}}(q,q_i)
  &=
  \cos\!\left(
    \operatorname{Emb}(\mathbf{Q}_{e,q}),
    \operatorname{Emb}(\mathbf{Q}_{e,i})
  \right),\\
s_{\mathrm{path}}(q,q_i)
  &=
  \operatorname{PathOverlap}
  (\hat{\mathbf{p}}_q,\mathbf{p}_i),\\
\operatorname{Sim}(q,q_i)
  &=
  \gamma_s s_{\mathrm{sem}}(q,q_i)
  +
  \gamma_p s_{\mathrm{path}}(q,q_i).
\end{aligned}
\end{equation}
where \(\gamma_p=0\) when no path sketch is used.
Let
\[
\mathcal{N}_{\theta}(q)
=
\{q_i:\operatorname{Sim}(q,q_i)\geq\theta_{\mathrm{sim}}\}
\]
be the retrieved training neighbours.
Their associated agents form the similarity-based candidate set
\begin{equation}
\mathcal{A}_{\mathrm{sim}}(q)
=
\bigcup_{q_i\in\mathcal{N}_{\theta}(q)}
\mathcal{A}(q_i).
\label{eq:similarity-agents}
\end{equation}
\begin{equation}
\label{eq:path-agents}
\mathcal{A}_{\mathrm{path}}(q)
=
\left\{
a_\ell :
\begin{array}{l}
\exists\, \tilde{p}\in\phi(\hat{\mathbf{p}}_q),
\ \exists\, \hat{D}_j\in\mathcal{G}_\ell,\\[-1mm]
\operatorname{Compatible}(\tilde{p},\hat{D}_j)=1
\end{array}
\right\}.
\end{equation}
Here, \(\operatorname{Compatible}(\tilde{p},\hat{D}_j)\) indicates that
partition \(\hat{D}_j\) contains entity--relation patterns compatible with
the predicted path fragment \(\tilde{p}\).
When a predicted path sketch is available, partitions containing compatible relation fragments additionally produce
\(\mathcal{A}_{\mathrm{path}}(q)\).
Let \(d(q)=1\) indicate that a retrieved neighbour exceeds the direct-routing
threshold \(\theta_{\mathrm{direct}}\).The activated agents are
\begin{equation}
\label{eq:activated-agents}
\mathcal{A}^{*}(q)
=
\begin{cases}
\mathcal{A}_{\mathrm{sim}}(q),
  & d(q)=1,\\
\mathcal{A}_{\mathrm{sim}}(q)
  \cup
  \mathcal{A}_{\mathrm{path}}(q),
  & d(q)=0.
\end{cases}
\end{equation}

Each selected agent receives a sub-task derived from the transferred question pattern or predicted relation fragments, and searches only within its assigned partitions.
For an agent sub-task \(t\) and candidate KG path \(p\), we use
\begin{equation}
\operatorname{Match}(p,t)
=
\frac{
|\operatorname{Ent}(p)\cap\operatorname{Ent}(t)|
}{
\max(1,|\operatorname{Ent}(t)|)
}.
\label{eq:path-match}
\end{equation}
Paths whose matching score exceeds \(\theta_{\mathrm{match}}\) contribute structured triples
\(\mathcal{T}_i\) and textualised evidence \(\mathcal{E}_i\).
All selected agents operate in parallel.
The implementation of relation-path sketch prediction,
\(\operatorname{PathOverlap}(\cdot,\cdot)\), and
\(\operatorname{Compatible}(\cdot,\cdot)\) is specified in
Appendix~\ref{app:routing-details}.
Algorithms~\ref{alg:ag_tq_route} and~\ref{alg:marag} in
Appendix~\ref{app:alg} provide the complete routing and parallel retrieval
procedures.

\subsection{Compatibility-Based Evidence Aggregation}
\label{sec:aggregation}

Let
\(
\mathcal{T}_{\mathrm{all}}=\bigcup_{a_i\in\mathcal{A}^{*}(q)}\mathcal{T}_i
\qquad\text{and}\qquad
\mathcal{E}_{\mathrm{all}}=\bigcup_{a_i\in\mathcal{A}^{*}(q)}\mathcal{E}_i
\)
denote the collected triples and textual evidence.
Each retrieved triple \(\tau\in\mathcal{T}_{\mathrm{all}}\) is assigned a support score:
\begin{equation}
\begin{adjustbox}{max width=\linewidth}
\(w(\tau)
=
\alpha \max_{i:\tau\in\mathcal{T}_i}\operatorname{Conf}(a_i)
+
\beta \operatorname{Freq}(\tau)
+
\gamma \operatorname{Entail}(\mathcal{E}_{\mathrm{all}},\tau)\)
\label{eq:triple-weight}
\end{adjustbox}
\end{equation}
where \(\operatorname{Conf}(a_i)\) is estimated from training-time retrieval success,
\(\operatorname{Freq}(\tau)\) measures multi-agent support, and
\(\operatorname{Entail}(\mathcal{E}_{\mathrm{all}},\tau)\) is produced by a lightweight evidence verifier.

To avoid incorrectly rejecting valid multi-valued KG relations, two triples are treated as conflicting only when they are explicitly contradictory or assign different objects to a relation identified as functional:
\begin{equation}
\label{eq:conflict}
\begin{adjustbox}{max width=\linewidth}
\(\operatorname{Conflict}(\tau_1,\tau_2)=1
\quad\text{iff}\quad
 \tau_1\models\neg\tau_2
  \;\vee\;
  \tau_2\models\neg\tau_1
 {}\vee\;
  \operatorname{FuncCol}(\tau_1,\tau_2).\)
\end{adjustbox}
\end{equation}
We construct a compatibility graph whose vertices are retrieved triples and whose edges connect non-conflicting pairs.
Rather than claiming exact maximum-weight clique inference, we apply a greedy maximum-weight compatible-set heuristic, iteratively retaining the highest-scoring triple compatible with the current evidence set.
The resulting set \(\mathcal{T}_{\mathrm{clean}}\) is supplied to the head agent:
\begin{equation}
\operatorname{Answer}(q)
=
\operatorname{LLM}_{\mathrm{head}}
\left(
q,\mathcal{T}_{\mathrm{clean}},\mathcal{E}_{\mathrm{all}}
\right).
\label{eq:final-answer}
\end{equation}

Implementation details for the support-score components,
functional-relation identification, and greedy compatible-set construction
are provided in Appendix~\ref{app:aggregation-details}.

\subsection{Conditional Evidence Preservation and Retrieval Cost}
\label{sec:theory}

The benefit of SPLIT-RAG depends on whether routing activates partitions containing the evidence required for a query.
We therefore state two conditional properties rather than assuming perfect coverage for arbitrary unseen questions.

\paragraph{Proposition 1: Evidence preservation under routing coverage.}
For a question \(q\), let \(\mathcal{T}^{*}(q)\subseteq\mathcal{E}\) be a sufficient set of KG triples for deriving a correct answer, and let
\[
\mathcal{U}(q)
=
\bigcup_{a_i\in\mathcal{A}^{*}(q)} \mathcal{G}_{a_i}
\]
be the union of activated partitions.
If
\[
\mathcal{T}^{*}(q)\subseteq \mathcal{U}(q)
\]
and the within-partition retriever returns all required triples present in \(\mathcal{U}(q)\), then routing does not remove any supporting triple required for answering \(q\) relative to retrieval over the full graph.

\paragraph{Proposition 2: Retrieval-cost reduction under selective activation.}
Let \(n=|\mathcal{E}|\) denote the total number of KG triples and let
\(n_q=|\mathcal{E}(\mathcal{U}(q))|\) denote the number of triples contained in the partitions activated for question \(q\).
For a retrieval operator with monotone search cost \(C(\cdot)\), SPLIT-RAG has online cost
\begin{equation}
\begin{adjustbox}{max width=\linewidth}
\(T_{\mathrm{split}}(q)
=
T_{\mathrm{plan}}(q)
+
C(n_q)
+
T_{\mathrm{agg}}(q),\quad
T_{\mathrm{full}}(q)
=
C(n).\)
\end{adjustbox}
\label{eq:retrieval-cost}
\end{equation}
Therefore, SPLIT-RAG reduces online retrieval cost whenever
\[
T_{\mathrm{plan}}(q)+C(n_q)+T_{\mathrm{agg}}(q)<C(n).
\]
Under a linear candidate-scanning model, this becomes
\[
T_{\mathrm{split}}(q)=T_{\mathrm{plan}}(q)+O(n_q),
\qquad
T_{\mathrm{full}}(q)=O(n),
\]
so the efficiency gain is governed by the activated-graph ratio \(n_q/n\) and the planner overhead.
We empirically measure both evidence coverage and end-to-end latency in Section~\ref{sec:experiments}.
Formal proofs are provided in Appendix~\ref{app:pr_eff}.

%%%%%%%%%%%%%%%%%%%%%%%%%%%%%%%%%%%%%%%%%%%%%%%%%%%%%%%%%%%%
\section{Experiments}
\label{sec:experiments}

\begin{table}[t]
  \centering
  \small
  \caption{Statistics of the KGQA benchmarks used in our experiments.}
  \label{table:datasets}
  \begin{tabular}{lrrc}
    \toprule
    Dataset & \#Train & \#Test & Max Hop \\
    \midrule
    WebQSP      & 2,826   & 1,628   & 2 \\
    CWQ         & 27,639  & 3,531   & 4 \\
    MetaQA-2Hop & 119,986 & 114,196 & 2 \\
    MetaQA-3Hop & 17,482  & 14,274  & 3 \\
    \bottomrule
  \end{tabular}
  \vspace{-1.0em}
\end{table}

\begin{table*}[h]
\centering
\caption{
Performance comparison on KGQA benchmarks.
Within the KG+LLM category, the \textbf{best} and
\underline{second-best} results are highlighted.
Embedding-based and LLM-only results are included as reference comparisons
where the corresponding metrics are available. KG+LLM baselines are reproduced based on unified Llama-3.1-8B head.
}
\label{table:baseline-results}
\begin{adjustbox}{max width=0.9\textwidth}
\begin{tabular}{l|ccc|ccc|cc|cc}
\toprule
\multicolumn{11}{c}{\textbf{Overall Results}} \\ \cline{1-11}
\multirow{2}{*}{Method} & \multicolumn{3}{c|}{WebQSP} & \multicolumn{3}{c|}{CWQ} & \multicolumn{2}{c|}{MetaQA-2Hop} & \multicolumn{2}{c}{MetaQA-3Hop} \\
\cmidrule(lr){2-4} \cmidrule(lr){5-7} \cmidrule(lr){8-9} \cmidrule(lr){10-11}
& Hit & H@1 & F1 & Hit & H@1 & F1 & Hit & H@1 & Hit & H@1 \\
\midrule

\rowcolor{black!10}\multicolumn{11}{l}{\textbf{Embedding methods}} \\

KV-Mem~\cite{miller2016key}
& \NA & 46.7 & \NA
& \NA & 21.1 & \NA
& \NA & 82.7
& \NA & 48.9 \\

GraftNet~\cite{sun2018open}
& \NA & 66.4 & \NA
& \NA & 32.8 & \NA
& \NA & 94.8
& \NA & 77.1 \\

PullNet~\cite{sun2019pullnet}
& \NA & 68.1 & \NA
& \NA & 47.2 & \NA
& \NA & \underline{99.9}
& \NA & 91.4 \\

EmbedKGQA~\cite{saxena2020improving}
& \NA & 66.6 & \NA
& \NA & 45.9 & \NA
& \NA & 98.8
& \NA & \underline{94.8} \\

TransferNet~\cite{shi2021transfernet}
& \NA & 71.4 & \NA
& \NA & 48.6 & \NA
& \NA & \textbf{100}
& \NA & \textbf{100} \\

\midrule
\rowcolor{black!10}\multicolumn{11}{l}{\textbf{LLM methods (no KG retrieval)}} \\

Llama3-8b
& 59.2 & \NA & \NA
& 33.1 & \NA & \NA
& 31.7 & \NA
& 28.9 & \NA \\

Davinci-003
& 48.3 & \NA & \NA
& 35.6 & \NA & \NA
& 42.5 & \NA
& 25.3 & \NA \\

ChatGPT
& 66.8 & \NA & \NA
& 39.9 & \NA & \NA
& 43.2 & \NA
& 31.0 & \NA \\

Gemini 2.0 Flash
& 65.3 & \NA & \NA
& 41.1 & \NA & \NA
& 46.9 & \NA
& 32.2 & \NA \\

Falcon-H1-7B-Instruct
& 67.4 & \NA & \NA
& 45.8 & \NA & \NA
& 48.1 & \NA
& 33.5 & \NA \\

\midrule
\rowcolor{black!10}\multicolumn{11}{l}{\textbf{KG+LLM methods}} \\

Graph-RAG
& 77.2 & 73.1 & 67.7
& 58.8 & 54.6 & 53.9
& 88.1 & 85.4
& 81.3 & 78.2 \\

RoG~\cite{luo2023reasoning}
& 85.7 & 80.0 & 70.8
& 62.6 & 57.8 & 56.2
& 93.8 & 90.1
& 87.9 & 84.8 \\

SubgraphRAG~\cite{li2024simple}
& 86.9 & 81.8 & 71.5
& 65.4 & 61.3 & 57.8
& 95.2 & 92.3
& 89.3 & 86.1 \\

ToG~\cite{sun2023think}
& 82.1 & 78.5 & 69.8
& 63.3 & 59.8 & 56.4
& 91.6 & 88.9
& 87.3 & 84.7 \\

GcR~\cite{luo2024graph}
& 84.5 & 79.7 & 68.9
& 61.2 & 57.4 & 55.1
& 93.5 & 91.3
& 89.7 & 86.4 \\

GNN-RAG~\cite{mavromatis2025gnn}
& \second{\underline{90.4}} & 85.3 & 73.5
& 69.5 & 65.1 & 62.2
& \second{\underline{96.8}} & \second{\underline{93.0}}
& 91.5 & 85.9 \\

RPO-RAG~\cite{um2026rpo}
& 89.1 & \second{\underline{87.0}} & \second{\underline{81.5}}
& \second{\underline{72.6}} & \second{\underline{68.4}} & \second{\underline{65.1}}
& 95.7 & 92.2
& \best{\textbf{93.4}}& \second{\underline{87.9}} \\

\textbf{SPLIT-RAG}
& \best{\textbf{92.8}} & \best{\textbf{89.2}} & \best{\textbf{82.2}}
& \best{\textbf{73.3}} & \best{\textbf{69.1}} & \best{\textbf{66.2}}
& \best{\textbf{97.1}} & \best{\textbf{95.4}}
& \best{\textbf{93.4}} & \best{\textbf{88.1}} \\
\midrule

\end{tabular}
\end{adjustbox}
\vspace{-1.0em}
\end{table*}

\subsection{Experimental Setup}
\label{sec:exp-setup}

\paragraph{Datasets.}
We evaluate SPLIT-RAG on four KGQA benchmarks:
WebQuestionsSP (WebQSP)~\cite{yih2016value},
Complex WebQuestions (CWQ)~\cite{talmor2018web},
and the 2-hop and 3-hop subsets of MetaQA~\cite{zhang2018variational}.
WebQSP and CWQ are constructed over Freebase and test reasoning over a
large-scale KG, while MetaQA evaluates controlled multi-hop
reasoning over a movie-domain ontology.
We omit MetaQA-1Hop as it does not require multi-step retrieval.
Dataset statistics are reported in Table~\ref{table:datasets}.

\paragraph{Baselines.}
We compare against three groups of methods:
(i) embedding-based KGQA models, including KV-Mem~\cite{miller2016key},
GraftNet~\cite{sun2018open}, PullNet~\cite{sun2019pullnet},
EmbedKGQA~\cite{saxena2020improving}, and
TransferNet~\cite{shi2021transfernet};
(ii) LLM-only prompting without KG retrieval; and
(iii) KG+LLM methods, including  RoG~\cite{luo2023reasoning},
SubgraphRAG~\cite{li2024simple}, ToG~\cite{sun2023think},
GcR~\cite{luo2024graph}, GNN-RAG~\cite{mavromatis2025gnn}, RPO-RAG~\cite{um2026rpo} and a standard graph-based RAG baseline.

\paragraph{Metrics and implementation.}
We report Hit, Hits@1 (H@1), and F1.
Hit measures whether any gold answer entity appears in the model output;
H@1 measures whether the top predicted entity is correct;
and F1 measures token-level overlap between predicted and gold answers.

\paragraph{Controlled comparison setting.}
 KG+LLM baselines are reproduced based on unified Llama-3.1-8B head model, identical benchmark splits, and the same
evaluation scripts. Latency comparisons further use matched retrieval and
decoding budgets, as detailed in Appendix~\ref{app:example}, Table~\ref{table:agent-budgets-webcwq} and~\ref{table:agent-budgets-metaqa}.

\subsection{Main Results}

Table~\ref{table:baseline-results} reports the overall KGQA performance.
SPLIT-RAG achieves the strongest results among KG+LLM methods across all four
benchmarks, including recent SOTA baselines such as GNN-RAG and
RPO-RAG.
The improvement is particularly relevant on WebQSP and CWQ, where the
underlying Freebase graph is substantially larger and more relation-rich than
MetaQA, supporting the effectiveness of question-conditioned partition routing
for large-scale KG retrieval.
On MetaQA, embedding-based methods remain highly competitive due to the small
and regular ontology; these datasets  are used primarily to analyze
controlled multi-hop behaviour and agent configurations.

\begin{table}[!h]
\vspace{-0.5em}
  \centering
  \caption{End-to-end latency per QA (ms) on WebQSP/CWQ under the unified budgets in Table~\ref{table:agent-budgets-webcwq}.}
  \label{tab:latency_tog_rog_split}
  \vspace{-1mm}
  \begin{adjustbox}{max width=\columnwidth}
  \begin{tabular}{l|cc}
    \toprule
    Method & WebQSP Avg Time (ms) & CWQ Avg Time (ms) \\
    \midrule
    Think-on-Graph (ToG) & 41586.2 & 47623.0 \\
    Reasoning on Graphs (RoG) & 34735.4 & 41344.2 \\
    GNN-RAG & 30542.9 & 35239.1 \\
    \textbf{SPLIT-RAG} & \textbf{29429.8} & \textbf{33726.8} \\
    \bottomrule
  \end{tabular}
  \end{adjustbox}
  \vspace{-2mm}
\end{table}

\begin{table}[!h]
\vspace{-0.5em}
  \centering 
  \caption{KG and Subgraph Size}
  \resizebox{\columnwidth}{!}{
  \begin{tabular}{lccccc}
    \toprule
    \textbf{Dataset} & \textbf{\#Entity} & \textbf{\#Relations} & \textbf{\#Triplets} & \textbf{\#Avg $\mathcal{G}$ Entity} & \textbf{ $\mathcal{G}$ Coverage}\\
    \midrule
    WebQSP       & 2,566,291    & 7,058 & 8,309,195 & 68421.9 & 91.6\\
    CWQ          & 2,566,291    & 7,058 & 8,309,195 & 40705.8 & 74.2\\
    MetaQA       & 43,234  & 9 & 133,582 & 8651.2 & 99.9\\
    \bottomrule
  \end{tabular}
  }
  \label{table:subgraphAttr}
\end{table}

%(Although it's impossible to achieve 100% accuracy on Meta QA dataset simply using RAG)

As shown in Table~\ref{tab:latency_tog_rog_split}, SPLIT-RAG achieves the
lowest end-to-end latency on both WebQSP and CWQ under matched budgets.
Importantly, it remains faster than the recent efficiency-oriented GNN-RAG
baseline, while also obtaining stronger QA performance in
Table~\ref{table:baseline-results}.
This result indicates that routing questions to a small set of specialised
partitions improves the accuracy--latency trade-off rather than merely reducing
retrieval cost.

Table~\ref{table:subgraphAttr} reports graph scale and routed-partition
coverage. MetaQA reaches near-complete coverage because it contains a compact
ontology with regular reasoning patterns. On the substantially larger
Freebase graph, SPLIT-RAG still achieves high coverage on WebQSP, while CWQ
remains more challenging because its questions involve more diverse
compositional relation patterns. These statistics help explain why CWQ
provides the most demanding setting for routing generalisation.

\subsection{Ablation Study} \label{ablation}

\begin{table}[!h]
\vspace{-0.5em}
  \centering 
  \caption{
    Component ablation on WebQSP.
    A removes similarity-based planning;
    B replaces partition-specialised multi-agent retrieval with a single agent;
    C removes compatibility-based evidence filtering.
}
  \resizebox{\columnwidth}{!}{
  \begin{tabular}{lccccc}
    \toprule
    \textbf{Method} & \textbf{Hit} & \textbf{H@1} & \textbf{F1} & \textbf{Avg.\#Parts} & \textbf{Avg.\ Search Size} \\
    \midrule
    SPLIT-RAG       & \textbf{92.8}   & \textbf{89.2}    & \textbf{82.2} & 4.8& 68,421.9\\ 
    SPLIT-RAG - A   & 85.8  &  83.6   & 74.6 & 6.5& 65,402.8\\
    SPLIT-RAG - B   & 69.7 &  63.9  & 58.5 & 1 & 303,071.1\\
    SPLIT-RAG - C   & \underline{87.6}  &  \underline{85.1} & \underline{79.1} & 4.8& 68,421.9  \\
    \bottomrule
  \end{tabular}
  }
  \label{table:ablation}
  \vspace{-0.5em}
\end{table}

\begin{table*}[h]
  \centering
  \caption{Full results on MetaQA-3Hop using different model tiers for agents (end-to-end per QA).}
  \label{table:diffAgents}
  \begin{adjustbox}{max width=\textwidth}
    \begin{tabular}{c|c|c|cc|c}
      \toprule
      Group & \(\mathcal{A}_i\) & \(\mathbf{A}_{\text{head}}\) & \textbf{Hit} & \textbf{H@1} & \textbf{Avg Time (s)}\\
      \midrule\midrule
      G1 & Gemini 2.0 Flash & Gemini 2.0 Flash-Lite & 82.1 & 80.3 & 28.1\\
      G2 & Gemini 2.5 Flash & Gemini 2.0 Flash-Lite & 84.5 & 82.8 & 35.3 \\
      G3 & Gemini 2.0 Flash-Lite & Gemini 2.0 Flash & 91.5 & 88.3 & 19.2\\
      G4 & Gemini 2.0 Flash & Gemini 2.0 Flash & 92.9 & 90.1 & 31.5\\
      G5 & Gemini 2.0 Flash-Lite & Gemini 2.5 Flash  & 93.7 & 91.0 & 27.6\\
      G6 & Gemini 2.0 Flash & Gemini 2.5 Flash & \textbf{95.4} & \textbf{91.7} & 34.2\\
      \bottomrule
    \end{tabular}
  \end{adjustbox}
\end{table*}

We conduct an ablation study on WebQSP to examine three components of
SPLIT-RAG:
(A) similarity-based retrieval planning using training-question patterns;
(B) partition-specialised multi-agent retrieval; and
(C) compatibility-based evidence filtering before head-agent generation.
For variant A, we replace question-similarity-based planning with coarse
subgraph-label matching.
For variant B, the selected partitions are merged into a single retrieval
region and searched by one agent.
For variant C, all retrieved triples are supplied directly to the head agent
without compatibility-based filtering.

Table~\ref{table:ablation} shows that all three components contribute to
SPLIT-RAG.
Removing similarity-based planning (A) reduces performance and increases the
average number of activated partitions from \(4.8\) to \(6.5\), indicating
that training-question patterns provide a more selective routing signal than
coarse subgraph labels alone.
Replacing partition-specialised multi-agent retrieval with a single agent (B)
causes the largest degradation, while expanding the average searched region
by more than four times.
Moreover, the union of selected partitions is disconnected for \(22.3\%\) of
questions in this variant, making it more difficult for a single retrieval
agent to recover evidence distributed across routed graph regions.
These results support the central design choice of searching
question-relevant partitions through specialised retrieval agents rather than
querying their merged union with one agent.
Finally, removing compatibility-based evidence filtering (C) yields a smaller
but consistent performance drop while leaving the retrieval space unchanged.
Although explicit incompatibilities are detected in only \(12.3\%\) of
evaluated questions, filtering them still improves answer performance,
supporting its role as a complementary post-retrieval refinement step. Definitions of these diagnostic statistics are provided in
Appendix~\ref{app:ablation-diagnostics}.

\begin{table*}[h]
\vspace{-0.8em}
  \centering
  \caption{Robustness to incomplete KGs on WebQSP. Edges are dropped uniformly at random before partitioning; both forward and inverse copies of a dropped edge are removed together.}
  \label{tab:incomplete-kg}
  \small
  \begin{tabular}{l|cccc}
    \toprule
    Setting & Hit & H@1 & F1 & Latency (s) \\
    \midrule
Full KG (0\% drop, Llama-3.1-8B head) & 92.8 & 89.2 & 82.2 & 29.4 \\
10\% dropped & 84.5 & 83.3 & 76.1 & 28.8 \\
30\% dropped & 72.3 & 69.2 & 65.5& 25.3 \\
50\% dropped & 59.2 & 57.1 & 55.9 & 19.7\\
    \bottomrule
  \end{tabular}
  \vspace{-0.5em}
\end{table*}

\subsection{Accuracy--Latency Tradeoff \& Model Flexibility: Is unifying all agents necessary?}
\label{section:diffAgent}

\begin{figure}[h]
  \centering
  \vspace{-0.5em}
  \includegraphics[width=0.86\columnwidth]{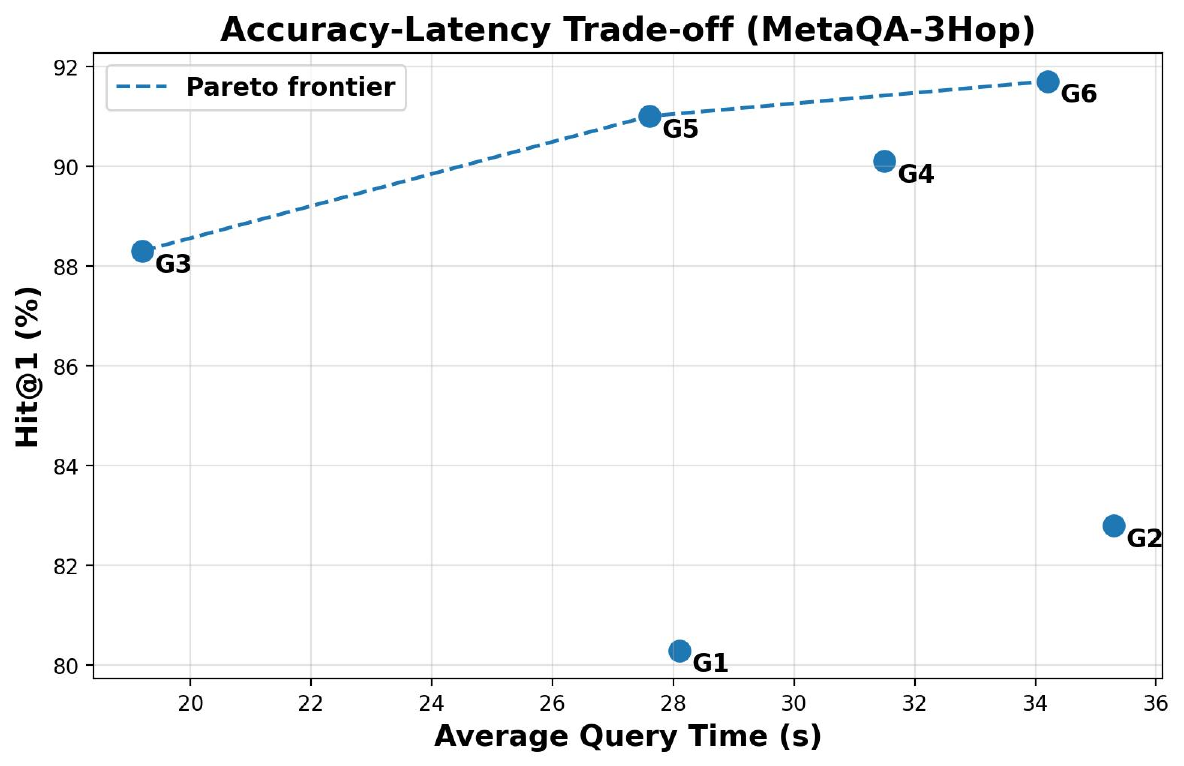}
  \vspace{-2mm}
  \caption{
    Accuracy--latency trade-off on MetaQA-3Hop measured by H@1.
    The Pareto frontier consists of G3, G5, and G6.
    }
  \label{fig:time-vs-hit}

  \vspace{2mm}

  \includegraphics[width=0.86\columnwidth]{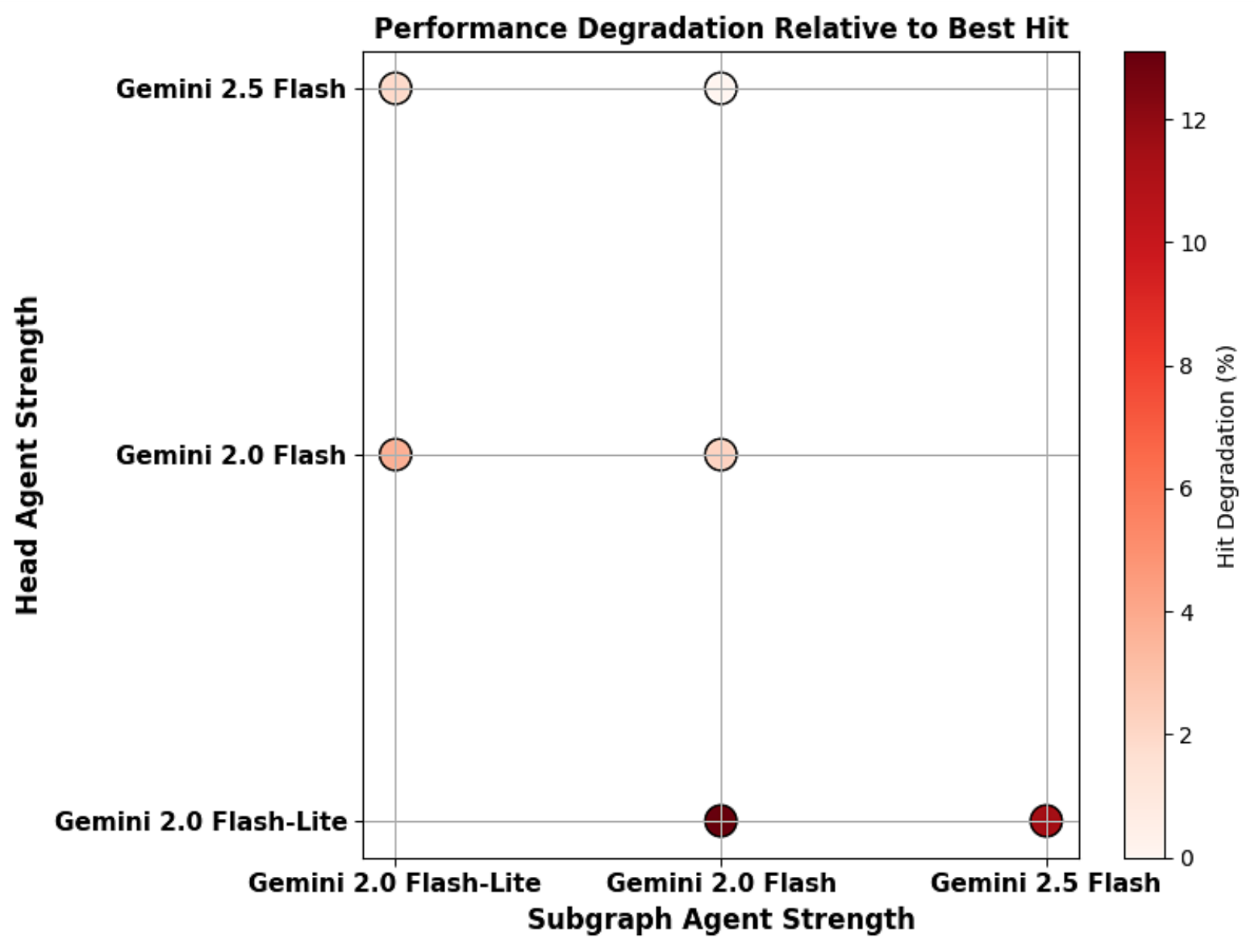}
  \vspace{-2mm}
  \caption{Performance degradation relative to the best H@1. Darker color indicates larger degradation.}
  \label{fig:degradation}
  \vspace{-0.5em}
\end{figure}

We further investigate whether all agents require the same model capacity by
varying the LLM tiers assigned to the subgraph agents
\(\{\mathcal{A}_i\}\) and the head agent
\(\mathbf{A}_{\mathrm{head}}\).
All configurations share the same planner, retriever, evidence caps, and
decoding budgets; only the model tiers vary.
The evaluated tiers are ordered by capacity as
Gemini 2.0 Flash-Lite \(<\) Gemini 2.0 Flash \(<\) Gemini 2.5 Flash.

As shown in Table~\ref{table:diffAgents}, configurations with a stronger head
agent consistently provide a better accuracy--latency trade-off than those
that allocate greater capacity only to subgraph agents.
In particular, replacing Flash subgraph agents with Flash-Lite while retaining
the strongest head model (G6 \(\rightarrow\) G5) saves 6.6 seconds per
question with only a small reduction in accuracy.
These results suggest that SPLIT-RAG \emph{does not require high-capacity models for
every partition-specific retrieval call}.A controlled Llama-3.1 check confirms the same direction: 8B workers with an 8B head obtain 89.2 H@1 at 31.6 s; 3B workers with an 8B head obtain
88.6 at 25.8 s; and 8B workers with a 3B head obtain 82.9 at 29.1 s.

\subsection{Robustness under Incomplete Knowledge Graphs}
\label{sec:incomplete-kg}

Table~\ref{tab:incomplete-kg} evaluates SPLIT-RAG when edges are removed
before partition construction, thereby testing both partition formation and
routed retrieval under incomplete evidence. The framework remains functional
under moderate knowledge loss, while accuracy decreases monotonically as
larger portions of answer-supporting paths are removed. Latency decreases
accordingly because the searchable graph becomes smaller, rather than because
the incomplete setting simplifies reasoning. This analysis characterises the
dependence of SPLIT-RAG on the availability of supporting KG evidence.

\section{Conclusion and Future Work}
\label{sec:conclusion}

We presented \textbf{SPLIT-RAG}, an efficiency-oriented KG-RAG framework that
constructs reusable question-aligned graph partitions, assigns specialised
retrieval agents to them, and routes each query to a small set of relevant
partition--agent groups.
A compatibility-based filtering step refines multi-agent evidence before
head-agent generation.
Experiments on WebQSP, CWQ, and MetaQA demonstrate that SPLIT-RAG is
competitive with recent KG+LLM retrieval systems while achieving lower
end-to-end latency under matched budgets.
Controlled ablations show that retaining
partition-scoped retrieval states supplies the quality benefit, whereas parallel worker execution supplies
the principal latency benefit. The model-tier analysis shows that lightweight retrieval workers can be effective when paired with a stronger head model. Additional
Additional incomplete-KG experiments characterise how retrieval quality changes
when supporting graph evidence becomes unavailable.
Future work will extend SPLIT-RAG to evolving and hybrid knowledge sources and
develop stronger verification mechanisms for ambiguous retrieved evidence.

\section*{Limitations}
\label{sec:limitations}

SPLIT-RAG is designed for question answering over structured, triple-based
knowledge graphs, where training questions provide useful signals for organising
retrieval regions. Its current evaluation focuses on KGQA benchmarks with
relatively stable schemas; extending the framework to hybrid corpora involving
text, tables, or dynamically updated sources remains an important direction for
future work. 

\section*{Ethics Considerations}

The experiments do not involve human
participants, private user data, or the collection of personally identifiable
information. SPLIT-RAG is intended as a research framework for improving the
accuracy--efficiency trade-off of knowledge-grounded question answering.
Although SPLIT-RAG
includes compatibility-based evidence filtering, this mechanism should not be
interpreted as a guarantee of factual correctness. 

\section*{Acknowledgments}

This research was partially supported by the Technology Innovation
Institute (TII), Abu Dhabi, United Arab Emirates, which conducts
cutting-edge applied research in artificial intelligence and advanced
technologies. Also, this research
used the NVIDIA H100 cluster provided by SHARON AI.

% --------------------------------------------------------------------------
%   Bibliography (style is set by acl.sty; do not override here)
% --------------------------------------------------------------------------
\bibliography{base}

% --------------------------------------------------------------------------
%   Appendix
% --------------------------------------------------------------------------

\appendix

\section{Notations}
\label{app:notation}

Important notations are listed in Table~\ref{table:notation}.

\section{Baseline Description}
\label{app:baseline}

Detailed descriptions of the baselines are provided in
Table~\ref{table:BaselinesDescribe}.

\section{Example of how training questions are preprocessed and stored} \label{app:training_questions_preprocessing}

Examples of different formats of questions are presented in Table~\ref{table:exampleQue}. Beside raw question, semantic context, entity-type context and path context are used in our experiments.

\begin{table*}[t]
\centering
\small

% ============================================================
% Table 1: Notations
% ============================================================
\caption{Important notations used in SPLIT-RAG.}
\label{table:notation}
\renewcommand{\arraystretch}{1.04}
\setlength{\tabcolsep}{5pt}
\begin{adjustbox}{max width=\textwidth}
\begin{tabular}{c|p{0.80\textwidth}}
\toprule
\textbf{Notation} & \textbf{Definition} \\
\midrule
\(F\) & The SPLIT-RAG framework. \\
\(\hat{\mathcal{G}}\) & Multi-agent generator for final answer synthesis. \\
\(D\) & Original knowledge graph. \\
\(\tau\) & Offline indexer that constructs graph partitions from \(D\) and \(\mathcal{Q}_{\mathrm{train}}\). \\
\(\psi\) & Query-time retriever that selects evidence from activated partitions. \\
\(\hat{\mathcal{K}}\) & Partitioned knowledge base constructed from \(D\). \\
\(\mathcal{Q}_{\mathrm{train}}\) & Training-question set used to guide graph partitioning and routing. \\
\(\mathbf{Q}_e\) & Entity-type representation of a question. \\
\(\mathbf{p}\) & Reasoning path associated with a question. \\
\(\mathcal{S}\) & Candidate graph partition. \\
\(C\) & Question-pattern cluster variable. \\
\(Z\) & Partition assignment variable of a reasoning-path fragment. \\
\(\hat{D}_i\) & The \(i\)-th graph partition. \\
\(\mathcal{G}_\ell\) & Set of partitions assigned to agent \(a_\ell\). \\
\(\mathcal{A}^{*}(q)\) & Set of agents activated for question \(q\). \\
\(\mathcal{T}_i\) & Structured triples retrieved by agent \(a_i\). \\
\(\mathcal{E}_i\) & Textualised evidence returned by agent \(a_i\). \\
\(\mathcal{T}_{\mathrm{clean}}\) & Compatible evidence set retained for final answer generation. \\
\bottomrule
\end{tabular}
\end{adjustbox}

\vspace{1.1em}

% ============================================================
% Table 2: Example of Training-Question Preprocessing
% ============================================================
\caption{Illustrative example of the processed representations of a training question stored in the question base.}
\label{table:exampleQue}
\renewcommand{\arraystretch}{1.06}
\setlength{\tabcolsep}{6pt}
\begin{adjustbox}{max width=\textwidth}
\begin{tabular}{l|p{0.78\textwidth}}
\toprule
\textbf{Question Format} & \textbf{Example} \\
\midrule
Raw Question &
Which genres contain films that share directors with
\texttt{[Black Snake Moan]}? \\

Semantic Context &
films share directors film \texttt{[Black Snake Moan]} genres \\

Entity-Type Context &
movie share director movie genre \\

Path Context &
\texttt{movie-[directed\_by]-director-[directed\_by]-\{Black Snake Moan\}-[has\_genre]-genre} \\
\bottomrule
\end{tabular}
\end{adjustbox}

\vspace{1.1em}

% ============================================================
% Table 3: Baseline Descriptions
% ============================================================
\caption{Descriptions of the baselines used in the experiments.}
\label{table:BaselinesDescribe}
\renewcommand{\arraystretch}{1.06}
\setlength{\tabcolsep}{5pt}
\begin{adjustbox}{max width=\textwidth}
\begin{tabular}{l|l|p{0.65\textwidth}}
\toprule
\textbf{Type} & \textbf{Method} & \textbf{Description} \\
\midrule
\multirow{5}{*}{Embedding}
& KV-Mem~\cite{miller2016key}
& Stores knowledge triples in a key--value memory network for multi-hop reasoning. \\

& GraftNet~\cite{sun2018open}
& Performs reasoning over question-specific KG subgraphs. \\

& PullNet~\cite{sun2019pullnet}
& Iteratively retrieves question-relevant subgraphs using graph neural networks. \\

& EmbedKGQA~\cite{saxena2020improving}
& Formulates KGQA as link prediction using learned entity and relation embeddings. \\

& TransferNet~\cite{shi2021transfernet}
& Propagates relevance scores over relation graphs for multi-hop answering. \\
\midrule
\multirow{5}{*}{KG+LLM}

& Graph-RAG
& Standard KG-based retrieval-augmented generation baseline. \\

& RoG~\cite{luo2023reasoning}
& Uses planning, graph retrieval, and reasoning for KG-based question answering. \\

& SubgraphRAG~\cite{li2024simple}
& Retrieves compact question-relevant KG subgraphs and uses them as context for LLM reasoning. \\

& ToG~\cite{sun2023think}
& Performs iterative LLM-guided graph exploration to construct reasoning paths over the KG. \\

& GcR~\cite{luo2024graph}
& Constrains LLM reasoning with graph-derived evidence to improve faithfulness in KGQA. \\

& GNN-RAG~\cite{mavromatis2025gnn}
& Uses graph-neural retrieval to select compact reasoning evidence for LLM-based KGQA. \\

& RPO-RAG~\cite{um2026rpo}
& Uses relation-aware preference optimisation to improve small-LLM reasoning over retrieved KG paths. \\

& SPLIT-RAG
& Uses question-guided graph partitioning with routed specialised retrieval agents. \\

\bottomrule
\end{tabular}
\end{adjustbox}

\end{table*}

%%%%%%%%%%%%%%%%%%%%%%%%%%%
\section{Algorithms}
\label{app:alg}

Algorithms~\ref{alg:alg_gp} to~\ref{alg:evidence-filtering} provide the implementation procedures for the main components
of SPLIT-RAG.

\clearpage

\begin{algorithm}[H]
\caption{Question-Guided KG Partitioning}
\label{alg:alg_gp}
\begin{algorithmic}[1]
\Require KG \(D=(\mathcal{V},\mathcal{E})\); training path fragments
         \(\tilde{\mathcal{P}}\); question-pattern labels \(C\);
         \(\lambda,\mu,\eta_{\max}\)
\Ensure Partitioned knowledge base
        \(\hat{\mathcal{K}}=\{\hat{D}_1,\ldots,\hat{D}_M\}\)

\State Initialise seed regions \(\mathcal{S}\) from fragments in
       \(\tilde{\mathcal{P}}\)
\Repeat
    \State \((s_a,s_b) \leftarrow
    \arg\max_{s_a,s_b\in\mathcal{S}}
    \Delta(s_a,s_b)\), where
    \Statex \hspace{\algorithmicindent}
    \(\Delta(s_a,s_b)
    =
    \mathcal{J}\!\left(
      \mathcal{S}\setminus\{s_a,s_b\}
      \cup\{s_a\cup s_b\}
    \right)
    -
    \mathcal{J}(\mathcal{S})\)
    \If{\(\Delta(s_a,s_b)>0\) \textbf{and}
        \(|\mathcal{V}_{s_a\cup s_b}|\leq\eta_{\max}\)}
        \State \(\mathcal{S}\leftarrow
        \mathcal{S}\setminus\{s_a,s_b\}
        \cup\{s_a\cup s_b\}\)
    \Else
        \State \textbf{break}
    \EndIf
\Until{no admissible merge improves Eq.~\eqref{eq:partition-objective}}

\For{each uncovered triple \(\tau\in\mathcal{E}\)}
    \State Assign \(\tau\) to its nearest compatible region in
    \(\mathcal{S}\)
\EndFor
\State Retain boundary entities required by cross-partition paths
\State \Return
\(\hat{\mathcal{K}}\leftarrow
\{\hat{D}_1,\ldots,\hat{D}_M\}\)
\end{algorithmic}
\end{algorithm}

\captionof{algorithm}{Question-Centric Agent--Subgraph Allocation}
\label{alg:alg_as}
\begin{algorithmic}[1]
\Require Partitions \(\hat{\mathcal{K}}=\{\hat{D}_1,\ldots,\hat{D}_M\}\);
         training questions \(\mathcal{Q}_{\mathrm{train}}\);
         capacity \(N_{\max}\);
         coherence threshold \(\theta_{\mathrm{coh}}\)
\Ensure Agent groups \(\{\mathcal{G}_1,\ldots,\mathcal{G}_L\}\)

\State Construct coverage sets \(\mathcal{C}_i\) via
       Eq.~\eqref{eq:coverage-set}
\State Initialise priority queue \(\mathcal{PQ}\) with
\Statex \hspace{\algorithmicindent}
\(\rho(\mathcal{C}_i)
=
\sum_j|\mathcal{C}_i\cap\mathcal{C}_j|
/
\sqrt{|\mathcal{C}_i|}\)
\State Initialise unassigned partitions
       \(\mathcal{U}\leftarrow\hat{\mathcal{K}}\), \(L\leftarrow0\)

\While{\(\mathcal{PQ}\neq\emptyset\) \textbf{and}
       \(\mathcal{U}\neq\emptyset\)}
    \State \(\mathcal{C}^{\star}
    \leftarrow
    \arg\max_{\mathcal{C}\in\mathcal{PQ}}
    \rho(\mathcal{C})\)
    \State \(\mathcal{G}_{\mathrm{cand}}
    \leftarrow
    \mathcal{C}^{\star}\cap\mathcal{U}\)

    \If{\(|\mathcal{G}_{\mathrm{cand}}|>N_{\max}\)}
        \State Retain the top-\(N_{\max}\) partitions ranked by
        \Statex \hspace{\algorithmicindent}
        \(f(\hat{D}_j)
        =
        \sum_i
        \mathbb{I}[\hat{D}_j\in\mathcal{C}_i]\)
    \EndIf

    \If{\(\operatorname{Coh}(\mathcal{G}_{\mathrm{cand}})
          \geq\theta_{\mathrm{coh}}\)}
        \State \(L\leftarrow L+1\);
        \(\mathcal{G}_L\leftarrow\mathcal{G}_{\mathrm{cand}}\)
        \State \(\mathcal{U}\leftarrow
        \mathcal{U}\setminus\mathcal{G}_L\)
        \State Remove covered candidate sets from \(\mathcal{PQ}\)
    \Else
        \State Remove \(\mathcal{C}^{\star}\) from \(\mathcal{PQ}\)
    \EndIf
\EndWhile

\For{each \(\hat{D}_j\in\mathcal{U}\)}
    \State \(\ell^{\star}\leftarrow
    \arg\max_{\ell:\,|\mathcal{G}_\ell|<N_{\max}}
    \operatorname{Sim}(\hat{D}_j,\mu_\ell)\)
    \State \(\mathcal{G}_{\ell^{\star}}
    \leftarrow
    \mathcal{G}_{\ell^{\star}}\cup\{\hat{D}_j\}\)
\EndFor
\State \Return \(\{\mathcal{G}_1,\ldots,\mathcal{G}_L\}\)
\end{algorithmic}

\begin{algorithm}[H]
\caption{Query-Time Agent Routing}
\label{alg:ag_tq_route}
\begin{algorithmic}[1]
\Require Question \(q\); training registry \(\mathcal{M}\);
         partition--agent groups \(\{\mathcal{G}_\ell\}_{\ell=1}^{L}\)
\Ensure Activated agent--task mapping
        \(\mathcal{MAP}=\{(a_\ell,t_\ell)\}\)

\State Form the entity-type representation \(\mathbf{Q}_{e,q}\)
\State Predict relation-path sketch \(\hat{\mathbf{p}}_q\) from \(q\)
       without using gold test paths or answers
\State Compute \(\operatorname{Sim}(q,q_i)\) using
       Eq.~\eqref{eq:routing-similarity}
\State Retrieve neighbours
       \(\mathcal{N}_{\theta}(q)\)
\State Construct \(\mathcal{A}_{\mathrm{sim}}(q)\) using
       Eq.~\eqref{eq:similarity-agents}
\State Construct \(\mathcal{A}_{\mathrm{path}}(q)\) using
       Eq.~\eqref{eq:path-agents}
\State Determine \(d(q)\) using threshold
       \(\theta_{\mathrm{direct}}\)
\State Activate \(\mathcal{A}^{*}(q)\) using
       Eq.~\eqref{eq:activated-agents}

\If{\(d(q)=1\)}
    \State Transfer sub-task patterns from the nearest matched
           training questions
\Else
    \State Construct sub-tasks from fragments in
           \(\phi(\hat{\mathbf{p}}_q)\)
\EndIf

\For{each activated agent \(a_\ell\in\mathcal{A}^{*}(q)\)}
    \State Assign its aligned sub-task \(t_\ell\) to
           \(\mathcal{MAP}\)
\EndFor

\State \Return \(\mathcal{MAP}\)
\end{algorithmic}
\end{algorithm}

\begin{algorithm}[H]
\caption{Parallel Subgraph Retrieval for Agent \(a_i\)}
\label{alg:marag}
\begin{algorithmic}[1]
\Require Sub-task \(t_i\); agent \(a_i\);
         assigned partitions \(\mathcal{G}_{a_i}\)
\Ensure Structured triples \(\mathcal{T}_i\);
        textual evidence \(\mathcal{E}_i\)

\State Retrieve matched paths:
\Statex \hspace{\algorithmicindent}
\(\mathcal{P}_i \leftarrow
\{\,p\in\textsc{Paths}(\mathcal{G}_{a_i}) :
\operatorname{Match}(p,t_i)
\geq\theta_{\mathrm{match}}\,\}\)
\Statex \hspace{\algorithmicindent}
\(\triangleright\) using Eq.~\eqref{eq:path-match}

\State Extract triples:
\Statex \hspace{\algorithmicindent}
\(\mathcal{T}_i
\leftarrow
\bigcup_{p\in\mathcal{P}_i}
\{\,\tau : \tau\in p\,\}\)

\State Textualise retrieved triples:
\Statex \hspace{\algorithmicindent}
\(\mathcal{E}_i
\leftarrow
\{\,\textsc{Textualize}(\tau)
: \tau\in\mathcal{T}_i\,\}\)

\State \Return \((\mathcal{T}_i,\mathcal{E}_i)\)
\end{algorithmic}
\end{algorithm}

\begin{algorithm}[H]
\caption{Compatibility-Based Evidence Filtering}
\label{alg:evidence-filtering}
\begin{algorithmic}[1]
\Require Retrieved triples \(\mathcal{T}_{\mathrm{all}}\);
         textual evidence \(\mathcal{E}_{\mathrm{all}}\);
         activated agents \(\mathcal{A}^{*}(q)\);
         functional relations \(\mathcal{R}_{\mathrm{func}}\)
\Ensure Filtered triple set \(\mathcal{T}_{\mathrm{clean}}\)

\For{each \(\tau\in\mathcal{T}_{\mathrm{all}}\)}
    \State Compute \(w(\tau)\) using Eq.~\eqref{eq:triple-weight}
\EndFor
\State Sort \(\mathcal{T}_{\mathrm{all}}\) in descending order of
       \(w(\tau)\)
\State Break ties by \(\operatorname{Freq}(\tau)\),
       \(\operatorname{Conf}(a_i)\), then lexical order
\State \(\mathcal{T}_{\mathrm{clean}}\leftarrow\emptyset\)

\For{each \(\tau\in\mathcal{T}_{\mathrm{all}}\) in sorted order}
    \If{\(\operatorname{Conflict}(\tau,\tau')=0\),
        \(\forall \tau'\in\mathcal{T}_{\mathrm{clean}}\)}
        \State
        \(\mathcal{T}_{\mathrm{clean}}
        \leftarrow
        \mathcal{T}_{\mathrm{clean}}\cup\{\tau\}\)
    \EndIf
\EndFor

\State \Return \(\mathcal{T}_{\mathrm{clean}}\)
\end{algorithmic}
\end{algorithm}

%%%%%%%%%%%%%%%%%%%%%%%%%%%

\section{Analysis of Partition Alignment and Retrieval Efficiency}
\label{app:pr_eff}

This section provides conditional properties of SPLIT-RAG.
The analysis does not assume that routing perfectly covers every unseen
question. Instead, it characterises (i) what is encouraged by the partition
objective, (ii) when routing preserves sufficient evidence, and
(iii) when selective activation reduces retrieval cost.

\begin{proposition}[Question-Pattern Alignment Objective]
\label{thm:info-interpretably}
Let \(C\) denote the question-pattern cluster associated with a sampled
training reasoning fragment, and let \(Z\) denote the partition assignment of
that fragment under a candidate partition \(\mathcal{S}\).
For the partition objective in Eq.~\eqref{eq:partition-objective},
\begin{equation}
\mathcal{J}(\mathcal{S})
=
I(C;Z)
-
\lambda \Omega_{\mathrm{size}}(\mathcal{S})
-
\mu |\mathcal{S}|,
\end{equation}
with \(\lambda,\mu\geq 0\), the induced mutual information satisfies
\begin{equation}
I(C;Z)
=
\mathcal{J}(\mathcal{S})
+
\lambda \Omega_{\mathrm{size}}(\mathcal{S})
+
\mu |\mathcal{S}|
\geq
\mathcal{J}(\mathcal{S}).
\label{eq:alignment-bound}
\end{equation}
Therefore, subject to the size and fragmentation penalties, maximising
\(\mathcal{J}(\mathcal{S})\) favours partitions whose assignments are
informative of question-pattern clusters.
\end{proposition}

\begin{proof}
By definition of the objective in Eq.~\eqref{eq:partition-objective},
\begin{equation}
\mathcal{J}(\mathcal{S})
=
I(C;Z)
-
\lambda \Omega_{\mathrm{size}}(\mathcal{S})
-
\mu |\mathcal{S}|.
\end{equation}
Rearranging gives
\begin{equation}
I(C;Z)
=
\mathcal{J}(\mathcal{S})
+
\lambda \Omega_{\mathrm{size}}(\mathcal{S})
+
\mu |\mathcal{S}|.
\end{equation}
Since
\(\lambda,\mu\geq 0\),
\(\Omega_{\mathrm{size}}(\mathcal{S})\geq 0\), and
\(|\mathcal{S}|\geq 0\), it follows that
\begin{equation}
I(C;Z)\geq \mathcal{J}(\mathcal{S}).
\end{equation}
Thus, after accounting for compactness and fragmentation penalties, improving
the objective encourages stronger dependence between question-pattern clusters
and partition assignments.
\end{proof}

\begin{proposition}[Evidence Preservation under Routing Coverage]
\label{thm:info-preserve}
For a question \(q\), let
\(\mathcal{T}^{*}(q)\subseteq\mathcal{E}\) be a sufficient set of KG triples
for deriving a correct answer.
Let
\begin{equation}
\mathcal{U}(q)
=
\bigcup_{a_i\in\mathcal{A}^{*}(q)}
\mathcal{G}_{a_i}
\end{equation}
denote the union of partitions activated by the router.
If
\begin{equation}
\mathcal{T}^{*}(q)\subseteq\mathcal{U}(q),
\label{eq:evidence-coverage-condition}
\end{equation}
and the local retrievers recover all triples in
\(\mathcal{T}^{*}(q)\) contained in \(\mathcal{U}(q)\), then SPLIT-RAG
preserves the supporting evidence required for answering \(q\) relative to
retrieval over the full graph.
\end{proposition}

\begin{proof}
By the routing coverage condition in
Eq.~\eqref{eq:evidence-coverage-condition}, every triple required by the
sufficient evidence set \(\mathcal{T}^{*}(q)\) lies inside the union of
activated partitions \(\mathcal{U}(q)\).
By assumption, the local retrievers recover every required triple from
\(\mathcal{U}(q)\).
Therefore,
\begin{equation}
\mathcal{T}^{*}(q)
\subseteq
\mathcal{T}_{\mathrm{all}},
\end{equation}
where \(\mathcal{T}_{\mathrm{all}}\) is the retrieved evidence supplied to the
aggregation module.
Hence, routing does not remove any triple necessary for deriving the answer
from \(\mathcal{T}^{*}(q)\).
\end{proof}

\begin{proposition}[Retrieval-Cost Reduction under Selective Activation]
\label{thm:complexity}
Let \(n=|\mathcal{E}|\) be the number of triples in the full KG, and let
\(n_q=|\mathcal{E}(\mathcal{U}(q))|\) be the number of triples contained in
the partitions activated for question \(q\).
Let \(C(x)\) be a monotone retrieval-cost function for searching over \(x\)
triples.
If \(T_{\mathrm{plan}}(q)\) and \(T_{\mathrm{agg}}(q)\) denote routing and
aggregation overhead, respectively, then
\begin{equation}
\begin{aligned}
T_{\mathrm{split}}(q)
&=
T_{\mathrm{plan}}(q)
+
C(n_q)
+
T_{\mathrm{agg}}(q),\\
T_{\mathrm{full}}(q)
&=
C(n).
\end{aligned}
\label{eq:retrieval-cost-comparison}
\end{equation}
Consequently, SPLIT-RAG reduces online retrieval time whenever
\begin{equation}
T_{\mathrm{plan}}(q)
+
C(n_q)
+
T_{\mathrm{agg}}(q)
<
C(n).
\label{eq:retrieval-saving-condition}
\end{equation}
\end{proposition}

\begin{proof}
Full-graph retrieval searches over all \(n\) triples and therefore incurs cost
\(C(n)\).
SPLIT-RAG first constructs a query-specific routing plan, retrieves only from
the activated graph union containing \(n_q\) triples, and aggregates the
resulting evidence.
Thus, its total online cost is
\begin{equation}
T_{\mathrm{split}}(q)
=
T_{\mathrm{plan}}(q)
+
C(n_q)
+
T_{\mathrm{agg}}(q).
\end{equation}
Comparing this expression with
\(T_{\mathrm{full}}(q)=C(n)\)
immediately yields Eq.~\eqref{eq:retrieval-saving-condition}.
Therefore, the efficiency gain is determined by the reduction in activated
graph size relative to routing and aggregation overhead.
\end{proof}

\paragraph{Parallel-agent execution.}
When activated agents retrieve evidence in parallel, let
\(n_{q,\ell}\) be the number of triples searched by agent \(a_\ell\).
Ignoring communication overhead already included in
\(T_{\mathrm{agg}}(q)\), the retrieval component of wall-clock latency becomes
\begin{equation}
C_{\mathrm{parallel}}(q)
=
\max_{a_\ell\in\mathcal{A}^{*}(q)}
C(n_{q,\ell}),
\end{equation}
rather than the sum of per-agent retrieval costs.
Accordingly, the parallel execution condition is
\begin{equation}
T_{\mathrm{plan}}(q)
+
\max_{a_\ell\in\mathcal{A}^{*}(q)}
C(n_{q,\ell})
+
T_{\mathrm{agg}}(q)
<
C(n).
\end{equation}
This condition motivates the latency evaluation reported in the main
experiments.

%%%%%%%%%%%%%%%%%%%%%%%%%%%%%%%%%%%%%%%%%%%%%%%%%%%%%%%%%%%%

\paragraph{Additional notes on fairness.}

All runs use the fixed decoding and retrieval budgets for MetaQA2-hop and 3-hop are in Table~\ref{table:agent-budgets-metaqa}. While for WebQSP and CWQ datasets the results are recorded in Table~\ref{table:agent-budgets-webcwq}
The wall-clock time is measured end-to-end per QA, averaging over the MetaQA-3Hop test split.

%%%%%%%%%%%%%%%%%%%%%%%%%%%%%%%%%%%%%%%%%%%%%%%%%%%%%%%%%%%%

\section{Path-Sketch Prediction and Routing Details}
\label{app:routing-details}

\paragraph{Relation-path sketch prediction.}
For each test question \(q\), SPLIT-RAG predicts a relation-path sketch
\(\hat{\mathbf{p}}_q\) before agent activation.
The predictor receives only the natural-language question, linked seed
entities, and the relation vocabulary available in the indexed KG; it does
not access gold answers or gold reasoning paths.
We implement the predictor using
\texttt{Llama-3.1-8B}.
The predictor is prompted to produce an ordered sequence of at most
\(H_{\max}\) relation labels, where \(H_{\max}\) is set to the maximum
reasoning depth of the corresponding benchmark.
Decoding uses temperature \texttt{0.1}, top-\(p\)
\texttt{0.9}, and retains the top-\(K\)
\texttt{3} path sketches.
The cost of sketch prediction is included in all reported end-to-end latency
measurements.

\paragraph{Relation vocabulary.}
The candidate vocabulary is constructed from relation labels appearing in the
indexed knowledge graph after preprocessing.
For efficiency, candidate relations may be restricted to those adjacent to
linked seed entities and those observed in training reasoning paths associated
with the retrieved question-pattern neighbours.
No relation labels derived from gold test paths are used.

\paragraph{Path-overlap score.}
Let \(\operatorname{Rel}(p)\) denote the ordered relation sequence of path
\(p\), and let \(\operatorname{LCS}(\cdot,\cdot)\) denote the length of their
longest common subsequence.
We define the path-overlap component in
Eq.~\eqref{eq:routing-similarity} as
\begin{equation}
\label{eq:path-overlap-detail}
\operatorname{PathOverlap}(\hat{\mathbf{p}}_q,\mathbf{p}_i)
=
\frac{
\operatorname{LCS}\!\left(
\operatorname{Rel}(\hat{\mathbf{p}}_q),
\operatorname{Rel}(\mathbf{p}_i)
\right)
}{
\max\!\left\{
1,\,
|\operatorname{Rel}(\hat{\mathbf{p}}_q)|
\right\}
}.
\end{equation}
This score rewards relation-order agreement while allowing a predicted sketch
to partially match a longer training path.

\paragraph{Partition compatibility.}
For a predicted path fragment \(\tilde{p}\) and a partition \(\hat{D}_j\),
we define
\begin{equation}
\begin{adjustbox}{max width=\linewidth}
\label{eq:compatible-detail}
\(\operatorname{Compatible}(\tilde{p},\hat{D}_j)
=
\mathbb{I}
\left[
\max_{p\in\mathcal{P}(\hat{D}_j)}
\operatorname{PathOverlap}(\tilde{p},p)
\geq
\theta_{\mathrm{path}}
\right],\)
\end{adjustbox}
\end{equation}
where \(\mathcal{P}(\hat{D}_j)\) is the set of indexed relation-path fragments
associated with partition \(\hat{D}_j\).
We set \(\theta_{\mathrm{path}}=\texttt{0.5}\) on the validation
split and keep it fixed for test evaluation.

\paragraph{Routing without path sketches.}
When no relation-path sketch is available, or when its prediction confidence
is below \(\theta_{\mathrm{pred}}\), the path term is disabled by setting
\(\gamma_p=0\), and routing relies only on semantic similarity between the
incoming question and stored training-question patterns.
The contribution of predicted path sketches is evaluated in
Table~\ref{tab:routing-ablation}.

%%%%%%%%%%%%%%%%%%%%

\section{Ablation}

\subsection{Additional Ablation Diagnostics}
\label{app:ablation-diagnostics}

To further interpret the component ablations in
Table~\ref{table:ablation}, we compute two diagnostic statistics on WebQSP.
For variant B, a routed retrieval region is counted as disconnected if the
union of its selected partitions contains more than one connected component
after merging. Under this definition, \(22.3\%\) of evaluated questions
produce disconnected merged regions.
For variant C, a question is counted as containing an explicit incompatibility
if at least one pair of retrieved triples satisfies
\(\operatorname{Conflict}(\tau_1,\tau_2)=1\)
before compatibility-based filtering. Explicit incompatibilities are detected
in \(12.3\%\) of evaluated questions.

\begin{table}[t]
  \centering
  \small
  \caption{
  Ablation of routing signals on WebQSP.
  Semantic denotes entity-type question similarity;
  Path denotes the predicted relation-path sketch.
  }
  \label{tab:routing-ablation}
  \begin{adjustbox}{max width=\columnwidth}
  \begin{tabular}{lccccc}
    \toprule
    Routing Signal & Hit & H@1 & F1 & Coverage & Time (s) \\
    \midrule
    Semantic only (\(\gamma_p=0\))
      & 90.6 & 86.9 & 79.7 & 88.4 & 29.8 \\
    Path only (\(\gamma_s=0\))
      & 88.1 & 83.8 & 76.5 & 82.0 & 30.5 \\
    Semantic + predicted path
      & 92.8 & 89.2 & 82.2 & 91.3 & 31.6> \\
    \bottomrule
  \end{tabular}
  \end{adjustbox}
\end{table}

%%%%%%%%%%%%%%%%%%%%%%%%%%%%%%%%%%%%%%%%%%%%%%%%%%%%%%%%%%%%

\section{Evidence Aggregation Details}
\label{app:aggregation-details}

\paragraph{Support-score components.}
For each retrieved triple \(\tau\), SPLIT-RAG computes the support score in
Eq.~\eqref{eq:triple-weight} using three normalised components.
The agent-confidence term is estimated from training-time retrieval success:
\begin{equation}
\label{eq:agent-confidence-detail}
\operatorname{Conf}(a_i)
=
\frac{
n^{\mathrm{succ}}_i + 1
}{
n^{\mathrm{ret}}_i + 2
},
\end{equation}
where \(n^{\mathrm{ret}}_i\) is the number of training questions routed to
agent \(a_i\), and \(n^{\mathrm{succ}}_i\) is the number for which the agent
retrieves at least one gold supporting triple.
The additive smoothing prevents unreliable confidence estimates for rarely
activated agents.

The multi-agent frequency component is defined as
\begin{equation}
\label{eq:frequency-detail}
\operatorname{Freq}(\tau)
=
\frac{
\sum_{a_i\in\mathcal{A}^{*}(q)}
\mathbb{I}[\tau\in\mathcal{T}_i]
}{
|\mathcal{A}^{*}(q)|
}.
\end{equation}
It rewards triples independently recovered from multiple activated agents.

\paragraph{Evidence entailment verifier.}
The entailment term evaluates whether a retrieved triple is supported by the
textualised evidence collected from activated agents:
\begin{equation}
\label{eq:entail-detail}
\operatorname{Entail}(\mathcal{E}_{\mathrm{all}},\tau)
=
P_{\mathrm{ver}}
\left(
\textsc{Supported}
\mid
\mathcal{E}_{\mathrm{all}},\tau
\right).
\end{equation}
We implement \(P_{\mathrm{ver}}\) using
\texttt{Llama-3.1-8b}
with temperature \texttt{0.0}.
The verifier receives the candidate triple and its associated textualised
evidence and returns a scalar support score in \([0,1]\).

\paragraph{Aggregation weights.}
The coefficients in Eq.~\eqref{eq:triple-weight} are selected on the
validation split and fixed for all test questions:
\begin{equation}
\label{eq:aggregation-weights}
\begin{adjustbox}{max width=\linewidth}
\(
\alpha=\texttt{0.5},
\qquad
\beta=\texttt{0.25},
\qquad
\gamma=\texttt{0.25}.
\)
\end{adjustbox}
\end{equation}
We tune these coefficients over
\(\texttt{{0,0.25,0.5,0.75,1.0}}\)
using validation-set H@1 as the selection criterion.

\paragraph{Functional relations.}
To avoid treating valid multi-valued relations as contradictions, object
collisions are considered conflicts only for relations estimated to be
functional.
For each relation \(r\), let
\(\deg^{+}_{r}(e)\) denote the number of distinct objects linked from entity
\(e\) through \(r\).
We construct the functional-relation set as
\begin{equation}
\label{eq:functional-relations}
\begin{adjustbox}{max width=\linewidth}
\(
\mathcal{R}_{\mathrm{func}}
=
\left\{
r:
\frac{
|\{e:\deg^{+}_{r}(e)\leq 1
\land
\deg^{+}_{r}(e)>0\}|
}{
|\{e:\deg^{+}_{r}(e)>0\}|
}
\geq
\theta_{\mathrm{func}}
\right\}.\)
\end{adjustbox}
\end{equation}
The threshold
\(\theta_{\mathrm{func}}=\texttt{0.85}\)
is set using the training KG only and is not derived from test answers.

For two triples
\(\tau_1=(e,r,o_1)\) and
\(\tau_2=(e,r,o_2)\),
functional collision is defined as
\begin{equation}
\label{eq:functional-collision-detail}
\operatorname{FuncCol}(\tau_1,\tau_2)
=
\mathbb{I}
\left[
r\in\mathcal{R}_{\mathrm{func}}
\land
o_1\neq o_2
\right].
\end{equation}

\paragraph{Greedy compatible-set construction.}
The retrieved triples are first sorted by support score
\(w(\tau)\) in descending order.
Ties are broken deterministically by
\(\operatorname{Freq}(\tau)\), then by
\(\operatorname{Conf}(a_i)\), and finally by lexical triple order.
A triple is retained only if it is compatible with every previously retained
triple.
Algorithm~\ref{alg:evidence-filtering} provides the complete procedure.

%%%%%%%%%%%%%%%%%%%%%%%%%%%%%%%%%%%%%%%%%%%%%%%%

\section{Experiment Data and Prompt Template}\label{app:example}

\begin{table}[h]
  \centering
  \caption{Fixed decoding \& retrieval budgets (WebQSP / CWQ).}
  \label{table:agent-budgets-webcwq}
  \vspace{-2mm}
  \begin{adjustbox}{max width=\columnwidth}
  \begin{tabular}{l|c|c}
    \toprule
    Setting & \(\{\mathcal{A}_i\}\) (subgraph agents) & \(\mathbf{A}_{\text{head}}\) \\
    \midrule
    Max input (tokens) per call & 3{,}000 & 4{,}096 \\
    Max new tokens & 96 & 256 \\
    Temperature / Top-$p$ & 0.2 / 0.9 & 0.3 / 0.95 \\
    Stop sequences & \texttt{\textbackslash n\textbackslash n}, KB-EOF & \texttt{\textbackslash n\textbackslash n}, Answer-EOF \\
    Tool-calls / function-calls & disabled & disabled \\
    Evidence cap to head (tokens) & --- & 3{,}200 \, \\
    Max subgraphs per question \(B\) & 4 & 4 \\
    Path depth / top-$K$ per subgraph & depth\(\le 2\), top-$K=8$ & --- \\
    Max triplets per agent & 600 & --- \\
    \bottomrule
  \end{tabular}
  \end{adjustbox}
  \vspace{-2mm}
\end{table}

\begin{table}[h]
  \centering
  \caption{Fixed decoding \& retrieval budgets (MetaQA 3-Hop).}
  \label{table:agent-budgets-metaqa}
  \vspace{-2mm}
  \begin{adjustbox}{max width=\columnwidth}
  \begin{tabular}{l|c|c}
    \toprule
    Setting & \(\{\mathcal{A}_i\}\) (subgraph agents) & \(\mathbf{A}_{\text{head}}\) \\
    \midrule
    Max input (tokens) per call & 3{,}000 & 4{,}096 \\
    Max new tokens & 96 & 256 \\
    Temperature / Top-$p$ & 0.2 / 0.9 & 0.3 / 0.95 \\
    Stop sequences & \texttt{\textbackslash n\textbackslash n}, KB-EOF & \texttt{\textbackslash n\textbackslash n}, Answer-EOF \\
    Tool-calls / function-calls & disabled & disabled \\
    Evidence cap to head (tokens) &  --- & 2{,}600 \, \\
    Max subgraphs per question \(B\) & 3 & 3 \\
    Path depth / top-$K$ per subgraph & depth\(\le 2\), top-$K=5$ & --- \\
    Max triplets per agent & 400 & --- \\
    \bottomrule
  \end{tabular}
  \end{adjustbox}
  \vspace{-2mm}
\end{table}

\subsection{Experiment compute resources}

Graph-partitioning and local inference with Llama-3-8B were executed on four NVIDIA RTX A5000 GPUs (24 GB VRAM each).
Calls to Gemini and GPT were made through remote APIs, so no local GPU was required; a CPU-only setup was sufficient.
All experiments ran on a server equipped with an Intel Xeon Silver 4310 (12 cores, 2.10 GHz) processor.

\subsection{Prompt Template}

\subsubsection{Subgraph Control Agent $\mathcal{A}_i$ Prompt}
\begin{Verbatim}[
  fontsize=\footnotesize,
  breaklines=true,
  breakanywhere=true,
  frame=single,
  framesep=2mm
]
Given Subgraph Context:
- Entity Types: {<T_entity>}
- Relation Types: {<T_relations>}
- Coverage: This subgraph focuses on <label> relationships
Given Current Subquestion: <Q_sub>

Task:
1. Analyze the subquestion's core information need
2. Generate {SPARQL/Cypher} query matching the subgraph schema

Critical Constraints:
- Use ONLY entities/relations from the subgraph context
- Query returns triples that directly answer the subquestion

Output Requirements:
{
  "query": "<generated_query>",
  "reasoning": "<brief_explanation_of_strategy>"
}
\end{Verbatim}

\subsubsection{Head Agent $\mathcal{A}_{H}$ Prompt}
\begin{Verbatim}[
  fontsize=\footnotesize,
  breaklines=true,
  breakanywhere=true,
  frame=single,
  framesep=2mm
]
Given verified facts: <T_clean>
Supporting evidence: <E_all>
Original question: <q_new>

Generate final answer with explanations,
resolving any remaining ambiguities.
\end{Verbatim}

\subsubsection{Relation-Path Sketch Predictor Prompt}
\begin{Verbatim}[
  fontsize=\footnotesize,
  breaklines=true,
  breakanywhere=true,
  frame=single,
  framesep=2mm
]
Given a knowledge-graph question and its linked seed entities,
predict a short relation-path sketch that is likely to connect
the seed entities to the answer.

Question:
<Q>

Linked Seed Entities:
<E_seed>

Candidate Relation Vocabulary:
<R_candidate>

Constraints:
1. Use only relations from the candidate relation vocabulary.
2. Produce at most <H_max> relations.
3. Do not generate answer entities.
4. Return relation paths only, ranked from most to least likely.

Output Format:
{
  "paths": [
    {
      "relations": ["<r_1>", "<r_2>", "..."],
      "confidence": <score>
    }
  ]
}
\end{Verbatim}

% =====================================================================
\section{Use of AI Assistants}
\label{app:ai-assistants}

We used a conversational AI assistant for help in writing and editing manuscript; no scientific and numerical claim in this paper was generated by an AI assistant.

\end{document}